\documentclass{article}       

\usepackage{amsmath,amssymb}
\usepackage{graphics,graphicx}
\usepackage{bm}
\usepackage{hyperref}
\usepackage{algorithm}
\usepackage{algorithmic}
\usepackage{verbatim}
\usepackage{color}
\usepackage{caption}
\usepackage{subcaption}

\newenvironment{proof}{ \medskip \noindent\emph{Proof} }{}
\newenvironment{acknowledgements}{\bigskip\bigskip\noindent {\bf Acknowledgements} \small }{}

\newtheorem{theorem}{Theorem}
\newtheorem{proposition}{Proposition}

\newcommand{\email}[1]{{E-mail: #1}}

\newcommand{\X}{{\bf X}}
\newcommand{\Y}{{\bf Y}}
\newcommand{\Xcal}{{\cal X}}
\newcommand{\Ycal}{{\cal Y}}
\newcommand{\Zcal}{{\cal Z}}
\newcommand{\Hcal}{{\cal H}}

\newcommand{\R}{\mathbb{R}}
\newcommand{\N}{\mathbb{N}}
\newcommand{\E}{\mathbb{E}}
\newcommand{\dpr}{\rangle}
\newcommand{\dpl}{\langle}

\newcommand{\independent}{\perp\mkern-11mu\perp}

\allowdisplaybreaks

\usepackage{mathptmx}      
\begin{document}

\title{\bf\Large Computing Functions of Random Variables via Reproducing Kernel Hilbert Space Representations}


\author{\normalsize ~~~~~~~~~~~~~~~~~~~~~~ Bernhard Sch\"olkopf\footnote{Max Planck Institute for Intelligent Systems, 72076 T\"ubingen, Germany} \and
        \normalsize Krikamol Muandet$^*$ ~~~~~~~~~~~~~~~~~~~~~~\and
        \normalsize Kenji Fukumizu\footnote{Institute for Statistical Mathematics, 10-3 Midori-cho, Tachikawa, Tokyo, Japan} \and
        \normalsize Jonas Peters\footnote{ETH Z\"urich, Seminar f\"ur Statistik, 8092 Z\"urich, Switzerland}}

\date{\today}

\maketitle

\begin{abstract}
We describe a method to perform functional operations on probability distributions of random variables. The method uses reproducing kernel Hilbert space representations of probability distributions, and it is applicable to all operations which can be applied to points drawn from the respective distributions. We refer to our approach as {\em kernel probabilistic programming}. We illustrate it on synthetic data, and show how it can be used for nonparametric structural equation models, with an application to causal inference.
\end{abstract}

\section{Introduction}
Data types, derived structures, and associated operations play a crucial role for programming languages and the computations we carry out using them. 
Choosing a data type, such as Integer, Float, or String, determines the possible values, as well as the operations that an object permits. Operations typically return their results in the form of data types. Composite or derived data types may be constructed from simpler ones, along with specialised operations applicable to them.

The goal of the present paper is to propose a way to represent distributions over data types, and to generalize operations built originally for the data types to operations applicable to those distributions. Our approach is nonparametric and thus not concerned with what distribution models make sense on which {\em statistical} data type (e.g., binary, ordinal, categorical). It is also general, in the sense that in principle, it applies to all data types and functional operations. The price to pay for this generality is that
\begin{itemize} 
\item our approach will, in most cases, provide approximate results only; however, we include a statistical analysis of the convergence properties of our approximations, and
\item for each data type involved (as either input or output), we require a positive definite kernel capturing a notion of similarity between two objects of that type.
Some of our results require, moreover, that the kernels be characteristic in the sense that they lead to injective mappings into associated Hilbert spaces.
\end{itemize}
In a nutshell, our approach represents distributions over objects as elements of a Hilbert space generated by a kernel, and describes how those elements are updated by operations available for sample points. If the kernel is trivial in the sense that each distinct object is only similar to itself, the method reduces to a Monte Carlo approach where the operations are applied to sample points which are propagated to the next step. Computationally, we represent the Hilbert space elements as finite weighted sets of objects, and all operations reduce to finite expansions in terms of kernel functions between objects.

The remainder of the present article is organized as follows. After describing the necessary preliminaries, we provide an exposition of our approach (Section~\ref{sec:method}). Section~\ref{sec:dependent-rvs} analyses an application to the problem of cause-effect inference using structural equation models. We conclude with a limited set of experimental results. 
 
\section{Kernel Maps}
\label{intro} 

\subsection{Positive definite kernels}
The concept of representing probability distributions in a reproducing kernel Hilbert space (RKHS) has recently attracted attention in statistical inference and machine learning \cite{Berlinet04:RKHS,Smola07Hilbert}. One of the advantages of this approach is that it allows us to apply RKHS methods to probability distributions, often with strong theoretical guarantees \cite{Sriperumbudur08injectivehilbert,Sriperumbudur10:Metrics}. It has been applied successfully in many domains such as graphical models \cite{Song10:HMM,Song11:KBP}, two-sample testing \cite{Gretton12:KTT}, domain adaptation \cite{Huang07:SSB,Gretton09:CSKMM,Muandet13:DG}, and supervised learning on probability distributions \cite{Muandet12:SMM,Szabo14:Regression}. We begin by briefly reviewing these methods, starting with some prerequisites.

We assume that our input data $\{x_1,\dots,x_m\}$ live in a nonempty set $\Xcal$ and are generated i.i.d.\ by a random experiment with Borel probability distribution $p$. By $k$, we denote a {\em positive definite kernel} on $\Xcal\times\Xcal$, i.e., a symmetric function
\begin{align}
k: \Xcal\times\Xcal &\to \R\\
(x,x') &\mapsto k(x,x')
\end{align}
satisfying the following nonnegativity condition: for any $m\in\N$, and $a_1,\dots,a_m\in\R$,
\begin{equation}\label{eq:pdkernel}
\sum_{i,j=1}^m a_ia_jk(x_i,x_j) \ge 0.
\end{equation}
If equality in \eqref{eq:pdkernel} implies that $a_1=\dots=a_m=0$, the kernel is called {\em strictly} positive definite.

\subsection{Kernel maps for points}
Kernel methods in machine learning, such as Support Vector Machines or Kernel PCA, are based on mapping the data into a reproducing kernel Hilbert space (RKHS) $\Hcal$ \cite{BosGuyVap92,SchSmo02,ShaCri04,HofSchSmo08,Steinwart08:SVM},
\begin{align}
\Phi: \Xcal &\to\Hcal\\
x &\mapsto \Phi(x),
\end{align}
where the {\em feature map} (or {\em kernel map}) $\Phi$ satisfies
\begin{equation}\label{eq:featuremap}
k(x,x') = \dpl\Phi(x),\Phi(x')\dpr
\end{equation}
for all $x,x'\in\Xcal$. One can show that every $k$ taking the form \eqref{eq:featuremap} is positive definite, and every positive definite $k$ allows the construction of $\Hcal$ and $\Phi$ satisfying \eqref{eq:featuremap}. The canonical feature map, which is what by default we think of whenever we write $\Phi$, is
\begin{align}
\Phi: \Xcal &\to \R^\Xcal\\
x &\mapsto k(x,.),\label{eq:kernelmap}
\end{align}
with an inner product satisfying the {\em reproducing kernel property}
\begin{equation}
k(x,x') = \dpl k(x,.),k(x',.)\dpr.
\end{equation}
Mapping observations $x\in\Xcal$ into a Hilbert space is rather convenient in some cases. If the original domain $\Xcal$ has no linear structure to begin with (e.g., if the $x$ are strings or graphs), then the Hilbert space representation provides us with the possibility to construct geometric algorithms by using the inner product of $\Hcal$. Moreover, even if $\Xcal$ is a linear space in its own right, it can be helpful to use a nonlinear feature map in order to construct algorithms that are linear in $\Hcal$ while corresponding to nonlinear methods in $\Xcal$.

\subsection{Kernel maps for sets and distributions}
One can generalize the map $\Phi$ to accept as inputs not only single points, but also sets of points or distributions. It was pointed out that the {\em kernel map of a set of points} $\X := \{ x_1,\dots,x_m\}$,
\begin{equation}\label{eq:meanmap}
\mu[\X] := \frac{1}{m} \sum_{i=1}^m \Phi(x_i),
\end{equation}
corresponds to a kernel density estimator in the input domain \cite{Scholkopf01:LKS,Sch_etal_NC_support}, provided the kernel is nonnegative and integrates to 1. However, the map \eqref{eq:meanmap} can be applied for all positive definite kernels, including ones that take negative values or that are not normalized. Moreover, the fact that $\mu[\X]$ lives in an RKHS and the use of the associated inner product and norm will have a number of subsequent advantages. For these reasons, it would be misleading to think of \eqref{eq:meanmap} simply as a kernel density estimator.

The {\em kernel map of a distribution} $p$ can be defined as the expectation of the feature map \cite{Berlinet04:RKHS,Smola07Hilbert,Gretton12:KTT},
\begin{equation}\label{eq:distrmap}
\mu[p] := {\E}_{x\sim p} [\Phi(x)],
\end{equation}
where we overload the symbol $\mu$ and assume, here and below, that $p$ is a Borel probability measure, and
\begin{equation}
{\E}_{x,x'\sim p}[k(x,x')]<\infty.
\end{equation}
A sufficient condition for this to hold is the assumption that there exists an $M\in\R$ such that
\begin{equation}
\|k(x,.)\| \le M < \infty,
\end{equation}
or equivalently $k(x,x)\le M^2$, on the support of $p$.
Kernel maps for sets of points or distributions are sometimes referred to as {\em kernel mean maps} to distinguish them from the original kernel map. Note, however, that they include the kernel map of a point as a special case, so there is some justification in using the same term. If $p=p_X$ is the law of a random variable $X$, we sometimes write $\mu[X]$ instead of $\mu[p]$.

In all cases it is important to understand what information we retain, and what we lose, when representing an object by its kernel map. We summarize the known results \cite{Steinwart08:SVM,FukGreSunSch08,Smola07Hilbert,Gretton12:KTT,Sriperumbudur10:Metrics} in Tables \ref{tab:1} and \ref{tab:2}.

\begin{table}
\begin{tabular}{|l|l|}
\hline
 $k(x,x') = \dpl x,x'\dpr$ & mean of $\X$\\
\hline
 $k(x,x') = (\dpl x,x'\dpr+1)^n$ & moments of $\X$ up to order $n\in\N$\\
\hline
$k(x,x')$ strictly p.d. & all of $\X$ (i.e., $\mu$ injective)\\
\hline
\end{tabular}
\caption{What information about a sample $\X$ does the kernel map $\mu[\X]$ (see \eqref{eq:meanmap}) contain?}
\label{tab:1}
\end{table}

\begin{table}
\begin{tabular}{|l|l|}
\hline
 $k(x,x') = \dpl x,x'\dpr$ & expectation of $p$\\
\hline
 $k(x,x') = (\dpl x,x'\dpr+1)^n$ & moments of $p$ up to order $n\in\N$\\
\hline
$k(x,x')$ characteristic/universal & all of $p$ (i.e., $\mu$ injective)\\
\hline
\end{tabular}
\caption{What information about $p$ does the kernel map $\mu[p]$ (see \eqref{eq:distrmap}) contain? For the notions of characteristic/universal kernels, see \cite{Steinwart02:IKC,FukGreSunSch08,FukBacJor09}; an example thereof is the Gaussian kernel \eqref{eq:gaussian}.}
\label{tab:2}
\end{table}

We conclude this section with a discussion of how to use kernel mean maps. To this end, first assume that $\Phi$ is injective, which is the case if $k$ is strictly positive definite (see Table~\ref{tab:1}) or characteristic/universal (see Table~\ref{tab:2}). Particular cases include the
{\em moment generating function} of a RV with distribution $p$,
\begin{equation}
M_p (.) = \E_{x\sim p}\left[ e^{\langle x,\;\cdot\;\rangle}\right],
\end{equation}
which equals \eqref{eq:distrmap} for $k(x,x')= e^{\langle x,x'\rangle}$ using \eqref{eq:kernelmap}.


We can use the map to test for equality of data sets,
\begin{equation}
\|\mu[\X]-\mu[\X']\|=0 \Longleftrightarrow \X=\X',
\end{equation}
or distributions,
\begin{equation}
\|\mu[p]-\mu[p']\|=0 \Longleftrightarrow p=p'.
\end{equation}
Two applications of this idea lead to tests for {\em homogeneity} and {\em independence}.
In the latter case, we estimate $\|\mu[p_xp_y]-\mu[p_{xy}]\|$ \cite{BacJor02,GreBouSmoSch05}; in the former case, we estimate $\|\mu[p]-\mu[p']\|$ \cite{Gretton12:KTT}.

Estimators for these applications can be constructed in terms of the empirical mean estimator (the kernel mean of the empirical distribution)
\begin{equation} \label{eq:empiricalmean}
 \mu[\hat p_m] = \frac{1}{m}\sum_{i=1}^m\Phi(x_i) = \mu[\X],
\end{equation}
where $\X = \{x_1,\ldots,x_m\}$ is an i.i.d.\ sample from $p$ (cf.\ \eqref{eq:meanmap}). As an aside, note that using ideas from James-Stein estimation \cite{Stein61:JSE}, we can construct shrinkage estimators that improve upon the standard empirical estimator (see e.g., \cite{MuandetFSGS2013,Muandet2014:KMSE}).

One can show that $\mu[\hat{p}_m]$ converges at rate $m^{-1/2}$ (cf. \cite{Smola07Hilbert} and \cite[Theorem 27]{Song08:Thesis}):
\begin{theorem}
\label{th:convergence}
Assume that $\|f\|_\infty\le 1$ for all $f\in\Hcal$ with $\|f\|_{\Hcal} \le 1$. Then with probability at least $1-\delta$,
\begin{equation}
\| \mu[\hat p_m] - \mu[p] \|_{\Hcal} \le \frac{2}{m}\mathbb{E}\left[\sqrt{\mathrm{tr}\, K }\right] + \sqrt{\frac{2\ln(2/\delta)}{m}} ,
\end{equation}
where $K_{ij}:=k(x_i,x_j)$.
\end{theorem}

Independent of the requirement of injectivity, $\mu$ can be used to compute expectations of arbitrary functions $f$ living in the RKHS, using the identity
\begin{equation}\label{eq:function-expectation}
\E_{x\sim p} [f(x)]= \dpl \mu[p],f \dpr,
\end{equation}
which follows from the fact that $k$ represents point evaluation in the RKHS,
\begin{equation}
f(x) = \dpl k(x,.),f\dpr.
\end{equation}
A small RKHS, such as the one spanned by the linear kernel
\begin{equation}
k(x,x')=\dpl x,x'\dpr,
\end{equation}
may not contain the functions we are interested in. If, on the other hand, our RKHS is sufficiently rich (e.g., associated with a universal kernel \cite{Steinwart02:IKC}), we can use \eqref{eq:function-expectation} to approximate, for instance, the probability of any interval $(a,b)$ on a bounded domain, by approximating the indicator function $I_{(a,b)}$ as a kernel expansion $\sum_{i=1}^n a_i k(x_i,.)$, and substituting the latter into \eqref{eq:function-expectation}. 
See \cite{KagFuk14} for further discussion. Alternatively, if $p$ has a density, we can estimate it using methods such as reproducing kernel moment matching and combinations with kernel density estimation \cite{SonZhaSmoGreetal08,KagFuk14}.

This shows that the map is not a one-way road: we can map our objects of interest into the RKHS, perform linear algebra and geometry on them \cite{SchSmo02}, and at the end answer questions of interest. In the next section, we shall take this a step further, and discuss how to implement rather general operations in the RKHS.

Before doing so, we mention two additional applications of kernel maps. We can map conditional distributions and perform Bayesian updates \cite{FukGreSunSch08,ZhangPJS2011,Fukumizu13a:KBR}, and we can connect kernel maps to Fourier optics, leading to a physical realization as Fraunhofer diffraction \cite{HarHirSch13}.

\section{Computing Functions of Independent Random Variables\label{sec:method}}
\subsection{Introduction and Earlier Work}

A random variable (RV) is a measurable function mapping possible outcomes of an underlying random experiment to a set $E$ (often, $E\subset\R^d$, but our approach will be more general). The probability measure of the random experiment induces the distribution of the random variable. We will below not deal with the underlying probability space explicitly, and instead directly start from random variables $X,Y$ with distributions $p_X,p_Y$ and values in $\Xcal,\Ycal$. Suppose we have access to (data from) $p_X$ and $p_Y$, and we want to compute the distribution of the random variable $f(X,Y)$, where $f$ is a measurable function defined on $\Xcal\times\Ycal$.

For instance, if our operation is addition $f(X,Y)=X+Y$, and the distributions $p_X$ and $p_Y$ have densities, we can compute the density of the distribution of $f(X,Y)$ by convolving those densities. If the distributions of $X$ and $Y$ belong to some parametric class, such as a class of distributions with Gaussian density functions, and if the arithmetic expression is elementary, then closed form solutions for certain favorable combinations exist. At the other end of the spectrum, we can resort to numerical integration or sampling to approximate $f(X,Y)$.

Arithmetic operations on random variables are abundant in science and engineering. Measurements in real-world systems are subject to uncertainty, and thus subsequent arithmetic operations on these measurements are operations on random variables. An example due to \cite{Springer79:Algebra} is signal amplification. Consider a set of $n$ amplifiers connected together in a serial fashion. If the amplification of the $i$-th amplifier is denoted by $X_i$, then the total amplification, denoted by $Y$, is $Y=X_1\cdot X_2\cdots X_n$, i.e., a product of $n$ random variables.


A well-established framework for arithmetic operation on independent random variables (iRVs) relies on \emph{integral transform} methods \cite{Springer79:Algebra}. The above example of addition already suggests that Fourier transforms may help, and indeed, people have used transforms such as the ones due to Fourier and Mellin to derive the distribution function of either the sum, difference, product, or quotient of iRVs \cite{Epstein48:Mellin,Springer1966:Product,Prasad1970:Algebraic,Springer79:Algebra}. \cite{Williamson89:PA} proposes an approximation using Laguerre polynomials, and a notion of \emph{envelopes} bounding the cumulative distribution function. This framework also allows for the treatment of dependent random variables, but the bounds can become very loose after repeated operations. \cite{Milios09:Algebraic} approximate the probability distributions of the input random variables as mixture models (using uniform and Gaussian distributions), and apply the computations to all mixture components.

\cite{JaroszewiczK12:Arithmetic} considers a numerical approach to implement arithmetic operations on iRVs, representing the distributions using piecewise Chebyshev approximations. This lends itself well to the use of approximation methods that perform well as long as the functions are well-behaved.
Finally, Monte Carlo approaches can be used as well, and are popular in scientific applications (see e.g., \cite{Ferson96:MonteCarlo}).

The goal of the present paper is to develop a derived data type representing a distribution over another data type, and to generalize the available computational operations to this data type, at least approximately. This would allow us to conveniently handle error propagation as in the example discussed earlier. It would also help us perform inference involving conditional distributions of such variables given observed data. The latter is the main topic of a subject area that has recently begun to attract attention, \emph{probabilistic programming} \cite{GHNR14}. A variety of probabilistic programming languages has been proposed \cite{Wood14:PP,Paige14:PP,Cas14:PP}. To emphasize the central role that kernel maps play in our approach, we refer to it as {\em kernel probabilistic programming (KPP)}.

\subsection{Computing Functions of Independent Random Variables using Kernel Maps}

The key idea of KPP is to provide a consistent estimator of the kernel map of an expression involving operations on random variables. This is done by applying the expression to the sample points, and showing that the resulting kernel expansion has the desired property. Operations involving more than one RV will increase the size of the expansion, but we can resort to existing RKHS approximation methods to keep the complexity of the resulting expansion limited, which is advisable in particular if we want to use it as a basis for further operations.
The benefits of KPP are three-fold. First, we do not make parametric assumptions on the distributions associated with the random variables. Second, our approach applies not only to real-valued random variables, but also to multivariate random variables, structured data, functional data, and other domains, as long as positive definite kernels can be defined on the data. Finally, it does not require explicit density estimation as an intermediate step, which is difficult in high dimensions. 

We begin by describing the basic idea. Let $f$ be a function of two independent RVs $X,Y$ taking values in the sets $\Xcal,\Ycal$, and suppose we are given i.i.d.\ $m$-samples $x_1,\dots,x_m$ and $y_1,\dots,y_m$. We are interested in the distribution of $f(X,Y)$, and seek to estimate its representation $\mu[f(X,Y)] := \E[\Phi(f(X,Y))]$ in the RKHS as
\begin{equation}\label{eq:simpleestimator}
\frac{1}{m^2} \sum_{i,j=1}^m \Phi\left(f (x_i,y_j)\right) .
\end{equation}
Although $x_1,\dots,x_m \sim p_X$ and $y_1,\dots,y_m \sim p_Y$ are i.i.d.\ observations, this does not imply that the $\{ f(x_i,y_j) | i,j=1,\dots,m\}$ form an i.i.d.\ $m^2$-sample from $f(X,Y)$, since --- loosely speaking --- each $x_i$ (and each $y_j$) leaves a footprint in $m$ of the observations, leading to a (possibly weak) dependency. Therefore, Theorem~\ref{th:convergence} does not imply that \eqref{eq:simpleestimator} is consistent. We need to do some additional work:
\begin{theorem}
  \label{thm:consistency}
Given two independent random variables $X,Y$ with values in $\Xcal,\Ycal$, mutually independent i.i.d.~samples $x_1,\ldots,x_m$ and $y_1,\ldots,y_n$, a measurable function $f:\Xcal\times\Ycal\to\Zcal$, and a positive definite kernel on $\Zcal\times\Zcal$ with RKHS map $\Phi$, then
\begin{equation}\label{eq:estimator}
	\frac{1}{mn}\sum_{i=1}^m\sum_{j=1}^n \Phi(f(x_i,y_j))
\end{equation}
is an unbiased and consistent estimator of $\mu[f(X,Y)]$.

Moreover, we have convergence in probability
\begin{equation}\label{eq:convergence_standard_estimator}
\left\| \frac{1}{mn}\sum_{i=1}^m\sum_{j=1}^n \Phi(f(x_i,y_j)) - \E[\Phi(f(X,Y))]\right\|
=O_p\left(\frac{1}{\sqrt{m}}+\frac{1}{\sqrt{n}}\right), \quad (m,n\to\infty).
\end{equation}
\end{theorem}
As an aside, note that \eqref{eq:estimator} is an RKHS valued two-sample U-statistic.

\begin{proof}
For any $i,j$, we have $\E[\Phi(f(x_i,y_j))] = \E[\Phi(f(X,Y))]$, hence \eqref{eq:estimator} is unbiased.

The convergence \eqref{eq:convergence_standard_estimator} can be obtained as a corollary to Theorem \ref{thm:consistency2}, and the proof is omitted here.
\end{proof}


\subsection{Approximating Expansions\label{sec:rs}}
To keep computational cost limited, we need to use approximations when performing multi-step operations. If for instance, the outcome of the first step takes the form \eqref{eq:estimator}, then we already have $m\cdot n$ terms, and subsequent steps would further increase the number of terms, thus quickly becoming computationally prohibitive.

We can do so by using the methods described in Chapter 18 of \cite{SchSmo02}. They fall in two categories. In reduced set {\em selection} methods, we provide a set of expansion points (e.g., all points $f(x_i,y_j)$ in \eqref{eq:estimator}), and the approximation method sparsifies the vector of expansion coefficients. This can be for instance done by solving eigenvalue problems or linear programs.  Reduced set {\em construction} methods, on the other hand, construct new expansion points. In the simplest case, they proceed by sequentially finding approximate pre-images of RKHS elements. They tend to be computationally more demanding and suffer from local minima; however, they can lead to sparser expansions.

Either way, we will end up with an approximation
\begin{equation}\label{rs-app}
\sum_{k=1}^p\gamma_k\Phi(z_k)
\end{equation}
of \eqref{eq:estimator}, where usually $p\ll m\cdot n$. Here, the $z_k$ are either a subset of the $f(x_i,y_j)$, or other points from $\Zcal$.

It is instructive to consider some special cases. For simplicity, assume that $\Zcal=\R^d$. If we use a Gaussian kernel
\begin{equation}\label{eq:gaussian}
k(x,x')=\exp(-\|x-x'\|^2/(2\sigma^2))
\end{equation}
whose bandwidth $\sigma$ is much smaller than the closest pair of sample points, then the points mapped into the RKHS will be almost orthogonal and there is no way to sparsify a kernel expansion such as \eqref{eq:estimator} without incurring a large RKHS error. In this case, we can identify the estimator with the sample itself, and KPP reduces to a Monte Carlo method. If, on the other hand, we use a linear kernel $k(z,z')=\dpl z,z'\dpr$ on $\Zcal=\R^d$, then $\Phi$ is the identity map and the expansion \eqref{eq:estimator} collapses to one real number, i.e., we would effectively represent $f(X,Y)$ by its mean for any further processing. By choosing kernels that lie `in between' these two extremes, we retain a varying amount of information which we can thus tune to our wishes, see Table~\ref{tab:1}.

\subsection{Computing Functions of RKHS Approximations}
More generally, consider approximations of kernel means $\mu[X]$ and $\mu[Y]$
\begin{equation}
  \label{eq:mean-estimators}
  \hat{\mu}[X] := \sum_{i=1}^{m'}\alpha_i\Phi_x(x'_i), \qquad
  \hat{\mu}[Y] := \sum_{j=1}^{n'}\beta_j\Phi_y(y'_j).
\end{equation}
In our case, we think of \eqref{eq:mean-estimators} as RKHS-norm approximations of the outcome of previous operations performed on random variables. Such approximations typically have coefficients $\bm{\alpha}\in\mathbb{R}^{n'}$ and $\bm{\beta}\in\mathbb{R}^{m'}$ 
that are not uniform, that may not sum to one, and that may take negative values \cite{SchSmo02}, e.g., for conditional mean maps \cite{Song10:KCOND,Fukumizu13a:KBR}.

We propose to approximate the kernel mean $\mu[f(X,Y)]$ by the estimator
\begin{equation}
  \label{eq:algebra-estimator}
  \hat{\mu}[f(X,Y)] := \frac{1}{\sum_{i=1}^{m'}\alpha_i\sum_{j=1}^{n'}\beta_j}\sum_{i=1}^{m'}\sum_{j=1}^{n'} \alpha_i\beta_j \Phi_z(f(x'_i,y'_j)),
\end{equation}
where the feature map $\Phi_z$ defined on $\Zcal$, the range of $f$, may be different from both $\Phi_x$ and $\Phi_y$. The expansion has $m'\cdot n'$ terms, which we can subsequently approximate more compactly in the form \eqref{rs-app}, ready for the next operation. Note that \eqref{eq:algebra-estimator} contains \eqref{eq:estimator} as a special case.

One of the advantages of our approach is that \eqref{eq:estimator} and \eqref{eq:algebra-estimator} apply for general data types. In other words, $\Xcal,\Ycal,\Zcal$ need not be vector spaces --- they may be arbitrary nonempty sets, as long as positive definite kernels can be defined on them.



\paragraph{Convergence analysis in an idealized setting}
We analyze the convergence of \eqref{eq:algebra-estimator} under the assumption that the expansion points are actually samples $x_1,\dots,x_m$ from $X$ and $y_1,\dots,y_n$ from $Y$, which is for instance the case if the expansions \eqref{eq:mean-estimators} are the result of reduced set selection methods (cf.\ Section~\ref{sec:rs}). Moreover, we assume that the expansion coefficients $\alpha_1,\dots,\alpha_m$ and $\beta_1,\dots,\beta_n$ are constants, i.e., independent of the samples. 

The following proposition gives a sufficient condition for the approximations in \eqref{eq:mean-estimators} to converge. Note that below, the coefficients $\alpha_1,\dots,\alpha_m$ depend on the sample size $m$, but for simplicity we refrain from writing them as $\alpha_{1,m},\dots,\alpha_{m,m}$; likewise, for $\beta_1,\dots,\beta_n$. We make this remark to ensure that readers are not puzzled by the below statement that $\sum_{i=1}^m \alpha_i^2 \to 0$ as $m\to\infty$.
\begin{proposition}\label{prop}
Let $x_1,\ldots,x_m$ be an i.i.d.~sample and $(\alpha_i)_{i=1}^m$ be constants with $\sum_{i=1}^m \alpha_i=1$.  Assume $\E[k(X,X)]>\E[k(X,\tilde{X})]$, where $X$ and $\tilde{X}$ are independent copies of $x_i$.  Then, the convergence 
\[
	\E\left\|\sum_{i=1}^m\alpha_i\Phi(x_i)-\mu[X]\right\|^2  \to 0\qquad (m\to\infty)
\]
holds true if and only if $\sum_{i=1}^m \alpha_i^2 \to 0$ as $m\to\infty$.
\end{proposition}
\begin{proof}
From the expansion
\begin{align*}
& \E\left\|\sum_{i=1}^m\alpha_i\Phi(x_i)-\mu[X]\right\|^2 \\
= & \sum_{i,s=1}^m \alpha_i\alpha_s \E[k(x_i,x_s)]-2\sum_{i=1}^m \alpha_i \E[k(x_i,X)] + \E[k(X,\tilde{X})] \\
= & \Bigl(1-\sum_i\alpha_i\Bigr)^2 \E[k(X,\tilde{X})] 
+ \Bigl( \sum_i\alpha_i^2\Bigr)\Bigl\{ \E[k(X,X)]-\E[k(X,\tilde{X})]
\Bigr\}, 
\end{align*}
the assertion is straightforward.
\end{proof}

The next result shows that if our approximations \eqref{eq:mean-estimators} converge in the sense of Proposition~\ref{prop}, then the estimator \eqref{eq:algebra-estimator} (with expansion coefficients summing to $1$) is consistent.

\begin{theorem}
  \label{thm:consistency2}
Let $x_1,\ldots,x_m$ and $y_1,\ldots,y_n$ be mutually independent i.i.d.~samples, and the constants $(\alpha_i)_{i=1}^m,(\beta_j)_{j=1}^n$ satisfy $\sum_{i=1}^m\alpha_i=\sum_{j=1}^n\beta_j=1$. Assume $\sum_{i=1}^m \alpha_i^2$ and $\sum_{j=1}^n \beta_j^2$ converge to zero 
as $n,m\to\infty$.  Then 
\[
\left\|\sum_{i=1}^m\sum_{j=1}^n\alpha_i\beta_j \Phi(f(x_i,y_j))-\mu[f(X,Y)]\right\| = O_p\left(\sqrt{\sum_i \alpha_i^2}+\sqrt{\sum_j\beta_j^2}\right)
\]
as $m,n\to\infty$.
\end{theorem}
\begin{proof}
By expanding and taking expectations, one can see that
\[
\E\left\| \sum_{i=1}^m\sum_{j=1}^n \alpha_i\beta_j \Phi(f(x_i,y_j)) - \E[\Phi(f(X,Y))]\right\|^2 
\]
equals
\begin{align*}
& \sum_{i=1}^m\sum_{j=1}^n \alpha_i^2 \beta_j^2 \E[k(f(X,Y), f(X,Y))]+\sum_{s\neq i}\sum_j\alpha_i\alpha_s\beta_j^2\E[k(f(X,Y), f(\tilde{X},Y))] \\
& \quad  +\sum_{i}\sum_{t\neq j}\alpha_i^2\beta_j\beta_t\E[k(f(X,Y), f(X,\tilde{Y}))]  \\
& \quad +\sum_{s\neq i}\sum_{t\neq j}\alpha_i\alpha_s\beta_j\beta_t\E[k(f(X,Y), f(\tilde{X},\tilde{Y}))] \\
& \quad -2 \sum_i\sum_j\alpha_i\beta_j \E[k(f(X,Y), f(\tilde{X},\tilde{Y}))] 
+\E[k(f(X,Y), f(\tilde{X},\tilde{Y}))] \\
= & \Bigl(\sum_i\alpha_i^2\Bigr)\Bigl(\sum_j\beta_j^2\Bigr)\E[k(f(X,Y), f(X,Y))] \\
& \quad + \left\{ \Bigl(1-\sum_i\alpha_i\sum_j\beta_j\Bigr)^2
+ \sum_i\alpha_i^2\sum_j\beta_j^2 \right. \\
& \qquad\qquad \left.- \sum_i\alpha_i^2(\sum_j\beta_j)^2 - (\sum_i\alpha_i)^2\sum_j\beta_j^2  \right\}\E[k(f(X,Y), f(\tilde{X},\tilde{Y}))] \\
& \quad + \Bigl((\sum_i\alpha_i)^2-\sum_i\alpha_i^2\Bigr)\Bigl(\sum_j\beta_j^2\Bigr)\E[k(f(X,Y), f(\tilde{X},Y))] \\
& \quad + \Bigl(\sum_i\alpha_i^2\Bigr)\Bigl((\sum_j\beta_j)^2-\sum_j\beta_j^2\Bigr)\E[k(f(X,Y), f(X,\tilde{Y}))] \\
= & 
\Bigl(\sum_i\alpha_i^2\Bigr)\Bigl(\sum_j\beta_j^2\Bigr)\E[k(f(X,Y), f(X,Y))] \\
& \quad + \left\{ \sum_i\alpha_i^2\sum_j\beta_j^2 - \sum_i\alpha_i^2 - \sum_j\beta_j^2  \right\}\E[k(f(X,Y), f(\tilde{X},\tilde{Y}))] \\
& \quad + \Bigl(1-\sum_i\alpha_i^2\Bigr)\Bigl(\sum_j\beta_j^2\Bigr)\E[k(f(X,Y), f(\tilde{X},Y))] \\
& \quad + \Bigl(\sum_i\alpha_i^2\Bigr)\Bigl(1-\sum_j\beta_j^2\Bigr)\E[k(f(X,Y), f(X,\tilde{Y}))],
\end{align*}
which implies the norm in the assertion of the theorem converges to zero at\\ $O_p\left(\sqrt{\sum_i \alpha_i^2}+\sqrt{\sum_j\beta_j^2}\right)$
under the assumptions on $\alpha_i$ and $\beta_j$.  Here $(\tilde{X},\tilde{Y})$ is an independent copy of $(X,Y)$.  This concludes the proof.
\end{proof}

Note that in the simplest case, where $\alpha_i=1/m$ and $\beta_j=1/n$, we have $\sum_i\alpha_i^2=1/m$ and $\sum_j\beta_j^2 = 1/n$, which proves Theorem \ref{thm:consistency}.  It is also easy to see from the proof that we do not strictly need $\sum_i\alpha_i=\sum_j\beta_j=1$ --- for the estimator to be consistent, it suffices if the sums converge to $1$. For a sufficient condition for this convergence, see \cite{KagFuk14}.




\paragraph{More general expansion sets} To conclude our discussion of the estimator \eqref{eq:algebra-estimator}, we turn to the case where the expansions \eqref{eq:mean-estimators} are computed by reduced set construction, i.e., they are not necessarily expressed in terms of samples from $X$ and $Y$. This is more difficult, and we do not provide a formal result, but just a qualitative discussion.

To this end, suppose the approximations \eqref{eq:mean-estimators} satisfy
\begin{align}
  \label{eq:mean-estimators-a}
\sum_{i=1}^{m'}\alpha_i=1 & \mbox{~and for all~}i, \alpha_i>0,
\\
  \label{eq:mean-estimators-b}
\sum_{j=1}^{n'}\beta_j=1 & \mbox{~and for all~}j, \beta_j>0,
\end{align}
and we approximate $\mu[f(X,Y)]$ by the quantity \eqref{eq:algebra-estimator}.

We assume that \eqref{eq:mean-estimators} are good approximations of the kernel means of two unknown random variables $X$ and $Y$; we also assume that $f$ and the kernel mean map along with its inverse are continuous. We have no samples from $X$ and $Y$, but we can turn \eqref{eq:mean-estimators} into sample estimates based on {\em artificial} samples $\X$ and $\Y$, for which we can then appeal to our estimator from Theorem~\ref{thm:consistency}.

To this end, denote by $\X' =( x'_1,\ldots,x'_{m'})$ and $\Y'=(y'_1,\ldots,y'_{n'})$ the expansion points in \eqref{eq:mean-estimators}. We construct a sample $\X =( x_1,x_2,\dots )$ whose kernel mean is close to $\sum_{i=1}^{m'}\alpha_i\Phi_x(x'_i)$ as follows: for each $i$, the point $x'_i$ appears in $\X$ with multiplicity $\lfloor m\cdot \alpha_i \rfloor$, i.e., the largest integer not exceeding $m\cdot \alpha_i$. This leads to a sample of size at most $m$. Note, moreover, that the multiplicity of $x'_i$, divided by $m$, differs from $\alpha_i$ by at most $1/m$, so effectively we have quantized the $\alpha_i$ coefficients to this accuracy.

Since $m'$ is constant, this implies that for any $\epsilon>0$, we can choose $m$ large enough to ensure that
\begin{equation}
\left\| \frac{1}{m}\sum_{1=1}^m \Phi_x(x_i) - \sum_{i=1}^{m'}\alpha_i\Phi_x(x'_i)\right\|^2 < \epsilon.
\end{equation}

We may thus work with $\frac{1}{m}\sum_{1=1}^m \Phi_x(x_i)$, which for strictly positive definite kernels corresponds uniquely to the sample $\X =( x_1,\ldots,x_m)$. 
By the same argument, we obtain a sample $\Y=(y_1,\ldots,y_n)$ approximating the second expansion. Substituting both samples into the estimator from Theorem~\ref{thm:consistency} leads to
\begin{equation}
  \label{eq:algebra-estimator-hat}
  \hat{\mu}[f(X,Y)] = \frac{1}{\sum_{i=1}^{m'}\hat{\alpha}_i\sum_{j=1}^{n'}\hat{\beta_j}}\sum_{i=1}^{m'}\sum_{j=1}^{n'} \hat{\alpha}_i\hat{\beta}_j \Phi_z(f(x'_i,y'_j)),
\end{equation}
where $\hat{\alpha}_i =  \lfloor m\cdot \alpha_i \rfloor / m$, and  $\hat{\beta}_i =  \lfloor n\cdot \beta_i \rfloor / n$.
By choosing sufficiently large $m,n$, this becomes an arbitrarily good approximation (in the RKHS norm) of the proposed estimator \eqref{eq:algebra-estimator}. Note, however, that we cannot claim based on this argument that this estimator is consistent, not the least since Theorem~\ref{thm:consistency} in the stated form requires i.i.d.\ samples.

\paragraph{Larger sets of random variables}
Without analysis, we include the estimator for the case of more than two variables:
Let $g$ be a measurable function of jointly independent RVs $U_j (j=1,\dots,p)$.
Given i.i.d.\ observations $u^j_1,\dots,u^j_m \sim U_j$, we have
\begin{equation} \label{eq:ustat}
\frac{1}{m^p} \sum_{m_1,\dots,m_p=1}^m \Phi\left(g (u^1_{m_1},\dots,u^p_{m_p})\right)
\xrightarrow{m\to\infty}
\mu\left[g(U_1,\dots,U_p )\right] .
\end{equation}
in probability.
Here, in order to keep notation simple, we have assumed that the samples sizes for each RV are identical.


As above, we note that (i) $g$ need not be real-valued, it can take values in some set $\Zcal$ for which we have a (possibly characteristic) positive definite kernel; (ii) we can extend this to general kernel expansions like \eqref{eq:algebra-estimator}; and (iii) if we use Gaussian kernels with width tending to 0, we can think of the above as a sampling method.

\section{Dependent RVs and Structural Equation Models}
\label{sec:dependent-rvs}

For dependent RVs, the proposed estimators are not applicable. One way to handle dependent RVs is to appeal to the fact that any joint distribution of random variables can be written as a structural equation model with independent noises.
This leads to an interesting application of our method to the field of causal inference.

We consider a model $X_i = f_i ( \mbox{PA}_i , U_i )$, for $i=1,\dots,p$, with jointly independent noise terms $U_1,\dots,U_p$. Such models arise for instance in causal inference \cite{Pearl2009}.
Each random variable $X_i$ is computed as a function $f_i$ of its noise term $U_i$ and its parents $\mbox{PA}_i$ in an underlying directed acyclic graph (DAG). Every graphical model w.r.t.\ a DAG can be expressed as such a structural equation model with suitable functions and noise terms (e.g., \cite{Peters2014anm}).

If we recursively substitute the parent equations, we can express each $X_i$ as a function of only the independent noise terms $U_1,\dots,U_p$,
\begin{equation}
X_i = g_i (U_1,\dots,U_p ).
\end{equation}
Since we know how to compute functions of independent RVs, we can try to test such a model (assuming knowledge of all involved quantities) by estimating the distance between RKHS images,
\begin{equation}
\Delta = \| \mu[X_i] - \mu[g_i (U_1,\dots,U_p )] \|^2
\end{equation}
using the estimator described in~\eqref{eq:ustat} (we discuss the bivariate case in Theorem~\ref{thm:additive-asymmetry}).
It may be unrealistic to assume we have access to all quantities. However, there is a special case where this is conveivable, which we will presently discuss. This is the case of additive noise models \cite{Peters2014anm}
\begin{equation}\label{eq:additive}
Y = f(X)+U, \qquad \mbox{~with~} X\independent U.
\end{equation}
Such models are of interest for cause-effect inference since it is known \cite{Peters2014anm} that in the generic case, a model \eqref{eq:additive} can only be fit in one direction, i.e., if \eqref{eq:additive} holds true, then we cannot simultaneously have
\begin{equation}\label{eq:additive2}
X = g(Y)+V, \qquad \mbox{~with~} Y\independent V.
\end{equation}

To measure how well \eqref{eq:additive} fits the data, we define an estimator
\begin{equation}\label{causal-delta}
\Delta_{emp}:=\left\| \frac{1}{m}\sum_{i=1}^m \Phi(y_i) - \frac{1}{m^2}\sum_{i,j=1}^m \Phi(f(x_i)+u_j)\right\|^2.
\end{equation}
Analogously, we define the estimator in the backward direction
\begin{equation}\label{causal-delta-bw}
\Delta_{emp}^{bw}:=\left\| \frac{1}{m}\sum_{i=1}^m \Phi(x_i) - \frac{1}{m^2}\sum_{i,j=1}^m \Phi(g(y_i)+v_j)\right\|^2.
\end{equation}
Here, we assume that we are given the conditional mean functions $f: x \mapsto \E[Y\,|\,X=x]$ and $g: y \mapsto \E[X\,|\,Y=y]$.

In practice, we would apply the following procedure: we are given a sample $(x_1,y_1),$ $\dots,(x_m,y_m)$. We estimate the function $f$ as well as the residual noise terms $u_1,\dots,u_m$ by regression, and likewise for the backward function $g$ \cite{Peters2014anm}. Strictly speaking, we need to use separate subsamples to estimate function and noise terms, respectively, see \cite{Kpotufe_etal_2014}.

Below, we show that $\Delta_{emp}$ converges to $0$ for additive noise models \eqref{eq:additive}. For the incorrect model \eqref{eq:additive2}, however, $\Delta_{emp}^{bw}$ will in the generic case not converge to zero. 
We can thus use the comparison of both values for deciding causal direction.
\begin{theorem}
  \label{thm:additive-asymmetry}
Suppose $x_1,\ldots,x_m$ and $u_1,\ldots,u_m$ are mutually independent i.i.d.\ samples, and $y_i=f(x_i) + u_i$. Assume further that the direction of the additive noise model is identifiable~\cite{Peters2014anm} and the kernel for $x$ is characteristic. We then have
\begin{align}
\Delta_{emp} &\to 0 \qquad \text{ and} \label{eqq1}\\
\Delta_{emp}^{bw} &\not \to 0 \label{eqq2}
\end{align}
in probability as $m\to\infty$.
\end{theorem}
\begin{proof}
Equation~\eqref{eqq1} follows from Theorem~\ref{thm:consistency} since
$\|\frac{1}{m}\sum_{i=1}^m \Phi(y_i) - \mu[Y]\| \to 0$ and
$
\|\frac{1}{m^2}\sum_{i,j=1}^{m^2} \Phi(f(x_i)+u_j) - \mu[Y]\| \to 0
$ (all convergences in this proof are in probability).

To prove \eqref{eqq2}, we assume that $\Delta_{emp}^{bw} \to 0$ which implies
\begin{equation} \label{eqqqq}
\left\| \frac{1}{m^2}\sum_{i,j=1}^m \Phi(g(y_i)+v_j) - \mu[X] \right\| \to 0.
\end{equation}
The key idea is to introduce a random variable $\tilde V$ that has the same distribution as $V$ but is independent of $Y$ and to consider the following decomposition of the sum appearing in~\eqref{eqqqq}:
\begin{align*}
\frac{1}{m^2}\sum_{i,j=1}^m \Phi(g(y_i)+v_j)
& =
\frac{1}{m^2} \sum_{i=1}^m \Phi(g(y_i)+v_i) \\
&\qquad \qquad +
\frac{1}{m^2} \sum_{i=1}^m \sum_{k=1}^{m-1} \Phi(g(y_i)+v_{i+k})
\\
& =
\frac{1}{m} \frac{1}{m}\sum_{i=1}^m \Phi(x_i) \\
&\qquad \qquad +
\frac{1}{m} \sum_{k=1}^{m-1} \frac{1}{m} \sum_{i=1}^m \Phi(g(y_i)+v_{i+k}) \\
&=: A_m + B_m,
\end{align*}
where the index for $v$ is interpreted modulo $m$, for instance, $v_{m+3} := v_3$.
Since $v_{i+k}=x_{i+k}-g(y_{i+k})$ is independent of $y_i$, 
it further follows from Theorem~\ref{thm:consistency} that
$\|A_m - \frac{1}{m}\mu[X]\| \to 0$  and
$\|B_m - \frac{m-1}{m}\mu[g(Y)+\tilde V]\| \to 0$. Therefore,
$$
\left\|A_m + B_m  - \frac{1}{m}\mu[X] - \frac{m-1}{m}\mu[g(Y)+\tilde V] \right\| \to 0.
$$
Together with~\eqref{eqqqq} this implies
$$
\left\|\mu[X] - \frac{1}{m}\mu[X] - \frac{m-1}{m}\mu[g(Y)+\tilde V] \right\| \to 0
$$
and therefore
$$
 \mu[g(Y)+\tilde V] = \mu[X].
$$
Since the kernel is characteristic, this implies
$$
g(Y) + \tilde V = X \quad \text{(in distribution),}
$$
with $Y\independent \tilde V$, which contradicts the identifiability of the additive noise model.
\end{proof}

As an aside, note that Theorem~\ref{thm:additive-asymmetry} would not hold if in~\eqref{causal-delta-bw} we were to estimate 
$\mu[g(Y)+V]$ by $\frac{1}{m} \sum_{i=1}^m \Phi(g(y_i)+v_i)$ instead of 
$\frac{1}{m^2} \sum_{i,j=1}^m \Phi(g(y_i)+v_j)$.



%

\section{Experiments}



\subsection{Synthetic data}

We consider basic arithmetic expressions that involve multiplication $X \cdot Y$, division $X/Y$, and exponentiation $X^Y$ on two independent scalar RVs $X$ and $Y$. Letting $p_X = \mathcal{N}(3,0.5)$ and $p_Y = \mathcal{N}(4,0.5)$, we draw i.i.d.\ samples $\mathbf{X}=\{x_1,\ldots,x_m\}$ and $\mathbf{Y}=\{y_1,\ldots,y_m\}$ from $p_X$ and $p_Y$. 

In the experiment, we are interested in the convergence (in RKHS norm) of our estimators to $\mu[f(X,Y)]$. Since we do not have access to the latter, we use an independent sample to construct a proxy $\hat{\mu}[f(X,Y)] = (1/\ell^2)\sum_{i,j=1}^{\ell}\Phi_z(f(x_i,y_j))$. We found that $\ell=100$ led to a sufficiently good approximation. 

Next, we compare the three estimators, referred to as $\mu_1,\mu_2$ and $\mu_3$ below, for sample sizes $m$ ranging from 10 to 50:
\begin{enumerate}
\item The sample-based estimator \eqref{eq:estimator}
\item The estimator \eqref{eq:algebra-estimator} based on approximations of the kernel means, taking the form $\hat{\mu}[X] := \sum_{i=1}^{m'}\alpha_i\Phi_x(x_i)$ and $\hat{\mu}[Y] := \sum_{j=1}^{m'}\beta_j\Phi_y(y_j)$ of $\mu[X]$ and $\mu[Y]$, respectively. We used the simplest possible reduced set selection method: we randomly subsampled subsets of size $m'\approx 0.4\cdot m$ from $\mathbf{X}$ and $\mathbf{Y}$, and optimized the coefficients $\{\alpha_1,\ldots,\alpha_{m'}\}$ and $\{\beta_1,\ldots,\beta_{m'}\}$ to best approximate the original kernel means (based on $\mathbf{X}$ and $\mathbf{Y}$) in the RKHS norm \cite[Section 18.3]{SchSmo02}.
\item Analogously to the case of one variable~\eqref{eq:empiricalmean}, we may also look at the estimator $\hat{\mu}_3[\X,\Y] := (1/m)\sum_{i=1}^m\Phi_z(f(x_i,y_i))$, which sums only over $m$ mutually independent terms, i.e., a small fraction of all terms of \eqref{eq:estimator}.
\end{enumerate}
For $i=1,2,3$, we evaluate the estimates $\hat{\mu}_i[f(X,Y)]$ using the error measure
\begin{equation}
  \label{eq:loss}
  L = \left\| \hat{\mu}_i[f(X,Y)] - \hat{\mu}[f(X,Y)] \right\|^2.
\end{equation} 
We use \eqref{eq:featuremap} to evaluate $L$ in terms of kernels.
In all cases, we employ a Gaussian RBF kernel \eqref{eq:gaussian} whose bandwidth parameter is chosen using the median heuristic, setting $\sigma$ to the median of the pairwise distances of distinct data points \cite{GreBouSmoSch05}.

Figure \ref{fig:synthetic-result} depicts the error \eqref{eq:loss} 
as a function of sample size $m$. For all operations, the error decreases as sample size increases. Note that $\Phi_z$ is different across the three operations, resulting in different scales of the average error in Figure \ref{fig:synthetic-result}.

\begin{figure*}[t!]
  \centering
  \includegraphics[width=\textwidth]{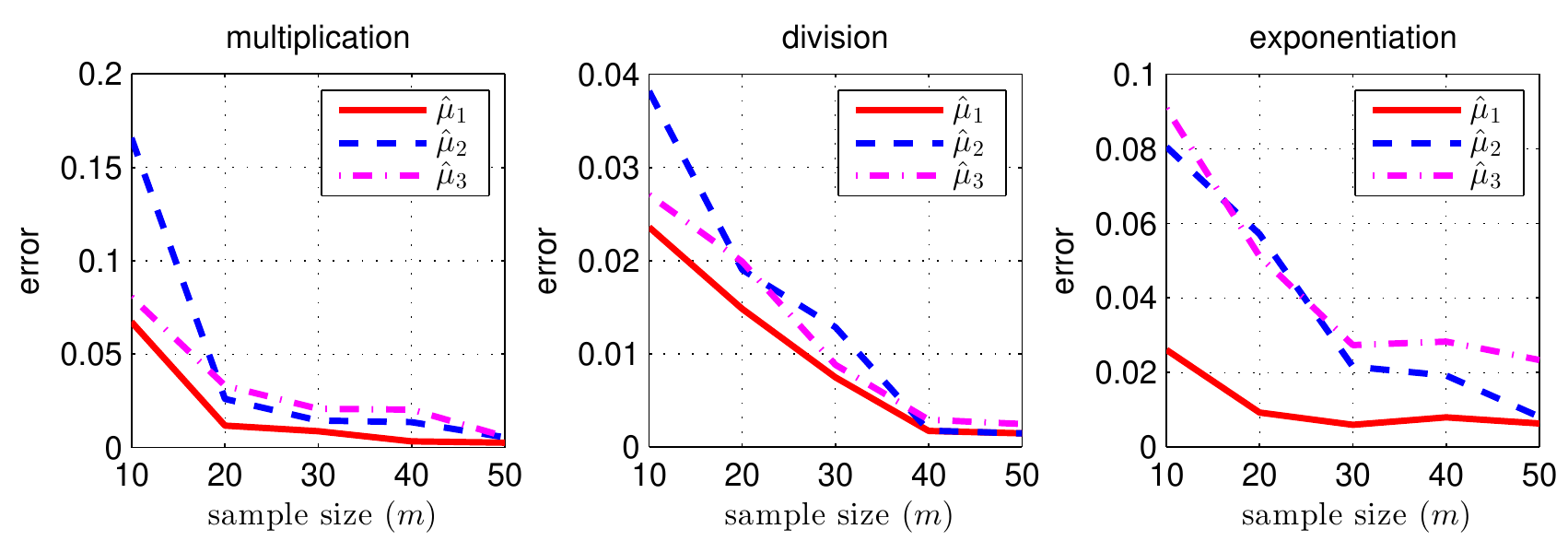}
  \caption{Error of the proposed estimators for three arithmetic operations ---  multiplication $X\cdot Y$, division $X/Y$, and exponentiation $X^Y$ --- as a function of sample size $m$. The error reported is an average of 30 repetitions of the simulations. The expensive estimator $\hat{\mu}_1$ (see \eqref{eq:estimator}) performs best. 
 The approximation  $\hat{\mu}_2$ (see \eqref{eq:algebra-estimator}) works well as sample sizes increase.
  }
  \label{fig:synthetic-result}
\end{figure*}

\subsection{Causal discovery via functions of kernel means}

We also apply our KPP approach to bivariate causal inference problem (cf. Section \ref{sec:dependent-rvs}). That is, given a pair of real-valued random variables $X$ and $Y$ with joint distribution $p_{XY}$, we are interested in identifying whether $X$ causes $Y$ (denote as $X\rightarrow Y$) or $Y$ causes $X$ (denote as $Y\rightarrow X$) based on observational data. We assume an additive noise model $E = f(C) + U$ with $C\independent U$ where $C,E,U$ denote cause, effect, and residual (or ``unexplained'') variable, respectively. Below we present a preliminary result on the \texttt{CauseEffectPairs} benchmark data set \cite{Mooij14:Benchmark}.

For each causal pair $(\mathbf{X},\mathbf{Y}) = \{(x_1,y_1),\ldots,(x_m,y_m)\}$, we estimate functions $y\approx f(x)$ and $x\approx g(y)$ as least-squares fits using degree 4 polynomials. We illustrate one example in Figure \ref{fig:tuebingen-pair}. 
Next, we compute the residuals in both directions as $u_i = y_i - f(x_i)$ and $v_j = x_j - g(y_j)$.\footnote{For simplicity, this was done using the same data; but cf.\ our discussion following \eqref{causal-delta}.} Finally, we compute scores $\Delta_{X\rightarrow Y}$ and $\Delta_{Y\rightarrow X}$ by
\begin{eqnarray*}
  \Delta_{X\rightarrow Y} &:=& \left\|\frac{1}{m}\sum_{i=1}^m\Phi(y_i) - \frac{1}{m^2}\sum_{i,j=1}^m\Phi(f(x_i) + u_j)\right\|^2 , \\
  \Delta_{Y\rightarrow X} &:=& \left\|\frac{1}{m}\sum_{i=1}^m\Phi(x_i) - \frac{1}{m^2}\sum_{i,j=1}^m\Phi(g(y_i) + v_j)\right\|^2 .
\end{eqnarray*}
Following Theorem \ref{thm:additive-asymmetry}, we can use the comparison between $\Delta_{X\rightarrow Y}$ and $\Delta_{Y\rightarrow X}$ to infer the causal direction. Specifically, we decide that $X\rightarrow Y$ if $\Delta_{X\rightarrow Y} < \Delta_{Y\rightarrow X}$, and that $Y\rightarrow X$ otherwise. In this experiment, we also use a Gaussian RBF kernel whose bandwidth parameter is chosen using the median heuristic. To speed up the computation of $\Delta_{X\rightarrow Y}$ and $\Delta_{Y\rightarrow X}$, we adopted a finite approximation of the feature map using 100 random Fourier features (see \cite{Rahimi07:RFF} for details).
We allow the method to abstain whenever the two values are closer than $\delta>0$. By increasing $\delta$, we can compute the method's accuracy as a function of a decision rate (i.e., the fraction of decisions that our method is forced to make) ranging from $100\%$ to $0\%$.

\begin{figure}
  \centering
  \begin{subfigure}[b]{0.48\textwidth}
    \includegraphics[width=\textwidth]{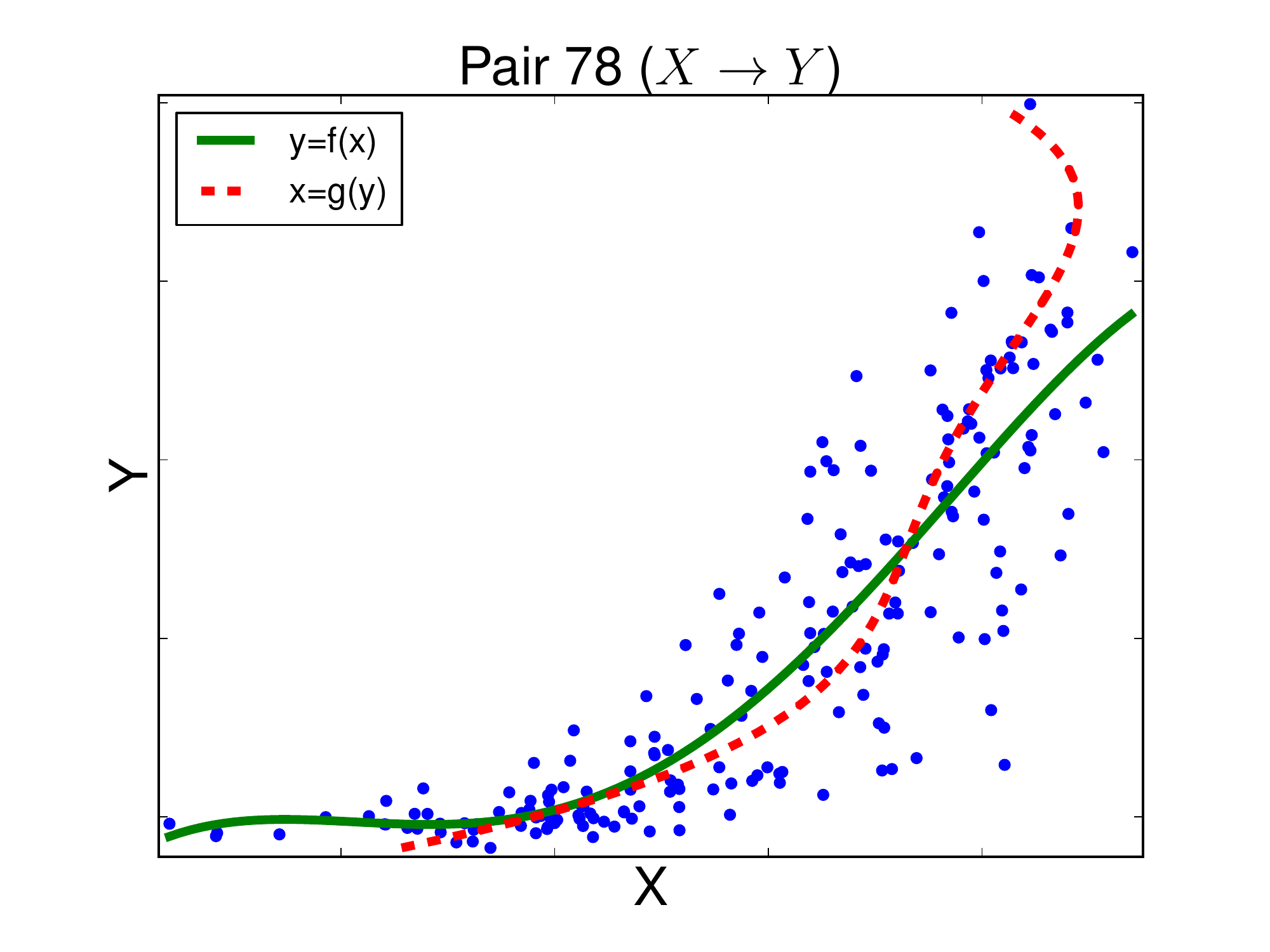} 
    \caption{Pair 78 and regressors $f,g$}
    \label{fig:tuebingen-pair}
  \end{subfigure}
  \begin{subfigure}[b]{0.48\textwidth}
    \includegraphics[width=\textwidth]{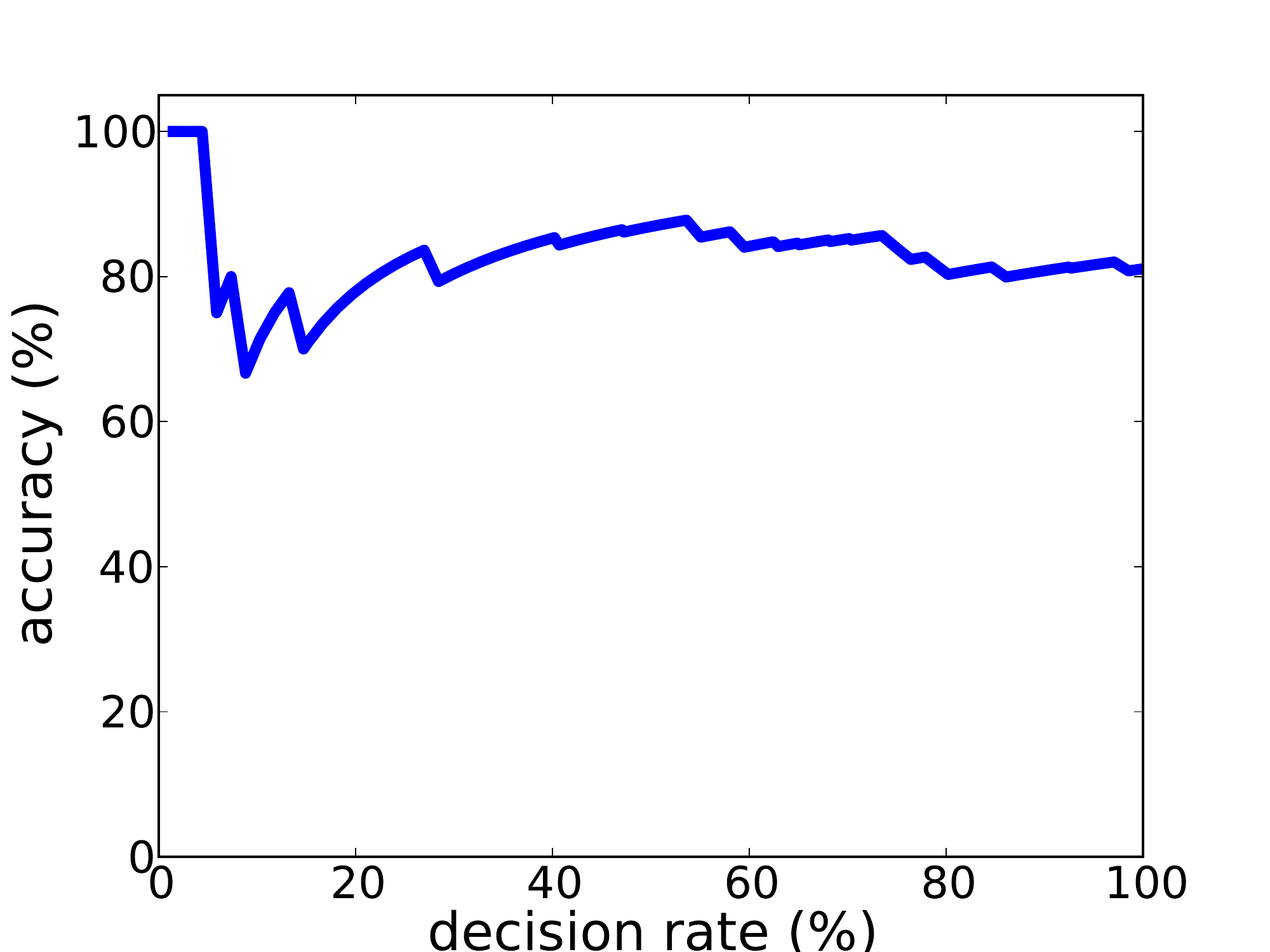} 
    \caption{the accuracy curve}
    \label{fig:tuebingen-result}
  \end{subfigure}
  \caption{(\subref{fig:tuebingen-pair}) Scatter plot of the data of causal pair 78 in the \texttt{CauseEffectPairs} benchmarks, along with the forward and backward function fits, $y=f(x)$ and $x=g(y)$. (\subref{fig:tuebingen-result}) Accuracy of cause-effect decisions on all the 81 pairs in the \texttt{CauseEffectPairs} benchmarks.}.
  \label{fig:cause-effect-res}
\end{figure} 

Figure \ref{fig:tuebingen-result} depicts the accuracy versus the decision rate for the $81$ pairs in the \texttt{CauseEffectPairs} benchmark collection. The method achieves an accuracy of 80\%, which is significantly better than random guessing, when forced to infer the causal direction of all 81 pairs. 


\section{Conclusions}\label{sec:conclu}
We have developed a kernel-based approach to compute functional operations on random variables taking values in arbitrary domains. We have proposed estimators for RKHS representations of those quantities, evaluated the approach on synthetic data, and showed how it can be used for cause-effect inference. While the results are encouraging, the material presented in this article only describes the main ideas, and much remains to be done. We believe there is significant potential for a unified perspective on probabilistic programming based on the described methods, and hope that some of the open problems will be addressed in future work.

\begin{acknowledgements}
Thanks to Dominik Janzing, Le Song and Ilya Tolstikhin for dicussions and comments.
\end{acknowledgements}

\bibliography{kme_algebra}   

%
%

\end{document}